\documentclass[fullpaper]{nldl}
\renewcommand{\DOI}{10.7557/18.6790}
\usepackage[utf8]{inputenc}
\usepackage{url}
\usepackage{graphicx}
\usepackage{authblk}
\usepackage{algorithm}
\usepackage{algorithmic}
\usepackage{booktabs}       % professional-quality tables
\usepackage[square,numbers]{natbib}
\usepackage{hyperref}

\usepackage{amssymb,amsthm,array,amsfonts}
\usepackage{bm}
\usepackage{underoverlap}
\newtheorem{theorem}{Theorem}
\usepackage{subfigure}

\newcommand\Sb{\mathbf S}
\newcommand\s{\bm s}

\newcommand\x{\bm{x}}
\newcommand\z{\bm z}
\newcommand\Lb{\mathbf L}
\newcommand\D{\mathbf D}
\newcommand\A{\mathbf A}
\newcommand\R{\mathbb R}
\newcommand\X{\mathbf X}
\newcommand\I{\mathbf I}
\renewcommand{\t}[1]{\tiny{#1}}

\usepackage{empheq}
\usepackage[most]{tcolorbox}

\newtcbox{\mymath}[1][]{%
    nobeforeafter,
    colframe=gray!20,
    colback=gray!20, 
    boxrule=.5pt,
    top=1mm,
    bottom=1mm,
    valign=center,
    halign=center,
    box align=center,
    #1}

\title{Simplifying Clustering with Graph Neural Networks}
\author{Filippo Maria Bianchi}
\affil{\small UiT the Arctic University of Norway and NORCE Norwegian Research Centre\\ \small \texttt{filippo.m.bianchi@uit.no}}

\date{\vspace{-5ex}}

\begin{document}
\nldlmaketitle

%%%%%%%%%%%%%%%%%%% POSSIBLE EXTENSIONS
% some "ablative" study should be conducted. Here, I think for instance of putting the proposed balancing term instead of the one in MinCutPool or DMoN (hence, the ablation is not really one; the JBGNN would be an "ablated" version of a modification of MinCutPool or DMoN). How would perform the modified versions w.r.t. the simplified one ?
%%%%%%%%%%%%%%%%%%% 

\begin{abstract}  
The objective functions used in spectral clustering are generally composed of two terms: i) a term that minimizes the local quadratic variation of the cluster assignments on the graph and; ii) a term that balances the clustering partition and helps avoiding degenerate solutions. 
This paper shows that a graph neural network, equipped with suitable message passing layers, can generate good cluster assignments by optimizing only a balancing term.
Results on attributed graph datasets show the effectiveness of the proposed approach in terms of clustering performance and computation time.
\end{abstract}

\section{Introduction}
Traditional clustering techniques partition the data directly in the input space by drawing regular boundaries to separate the clusters.
This makes them unsuitable to handle complex data structures, such as images or time series, which lie in high-dimensional spaces where the relationships between samples are highly non-linear. 
Deep learning techniques allow to transform data samples into suitable representations, which can partitioned into meaningful clusters~\citep{zhou2022comprehensive}.
Remarkably, end-to-end deep learning frameworks can directly map complex data directly into their cluster assignments~\citep{kampffmeyer2019deep}.

Of particular interest for this work, are those data characterized by relationships, or interactions, among samples that are described by a graph. 
Graph Neural Networks (GNNs) are deep learning architectures specifically designed to process and make inference on such data~\citep{hamilton2020graph}.
Recently, GNNs have been adopted to cluster the nodes of an attributed graph based on their features and the graph topology.
Inspired by spectral clustering algorithms, such GNNs optimize an unsupervised loss composed of two terms: the first ensures that connected nodes are assigned to the same cluster; the second is a balancing term, which prevents degenerate solutions both by encouraging the samples to be assigned to only one cluster and the clusters to have similar size~\citep{bianchi2020spectral, tsitsulin2020graph}.

In this work, I considerably simplify the clustering objective optimized by the previous GNN models by introducing a minimalist unsupervised loss, which consists only of a balancing term.
The proposed loss is used to train a GNN composed of standard message passing layers that operate on a particular connectivity matrix.
The empirical evaluation shows that the proposed model significantly reduces the computational complexity, while maintaining competitive clustering performance.

%%%%%%%%%%%%%%%%%%%%%%%%%%%%%%%%%%%%%%%%%
% BACKGROUND
%%%%%%%%%%%%%%%%%%%%%%%%%%%%%%%%%%%%%%%%%
\section{Background}
Let a graph be represented by a tuple $G = \{ \mathcal{V}, \mathcal{E} \}$, with node set $\mathcal{V}$ and edge set $\mathcal{E}$. 
Let $|\mathcal{V}|=N$ and $|\mathcal{E}|=E$ be the number of nodes and edges, respectively.
Each node $i$ is associated with a feature vector $\x_i \in \mathbb{R}^F$.
A graph is conveniently described by its adjacency matrix $\A \in \mathbb{R}^{N \times N}$ and the node features matrix $\X \in \R^{N \times F}$.

%% -------------------------------------------------------
\subsection{Spectral Clustering}
Graph clustering aims at partitioning the nodes in $K$ subsets, so that the similarity between nodes in the same subset is maximized.
The most famous graph clustering approach is spectral clustering, which relies on the $k-$way mincut objective to find a partition that minimizes the volume of edges crossing different clusters~\citep{von2007tutorial}.
To avoid degenerate solutions, the objective function includes a balancing term that penalizes partitions where clusters have very unequal sizes.
Specifically, the balanced $K$-cut objective can be defined as a ratio of two set functions:
\begin{equation}
\label{eq:balanced_k_cut}
    \min_{C_1, \dots, C_K} \sum_{k=1}^K \frac{\text{cut}(C_k, \bar{C}_k)}{\hat{B}(C_k)},
\end{equation}
where $\hat{B}(\cdot)$ is a set function that balances the size of the clusters in the partition.
Depending on the choice of $\hat{B}(\cdot)$, one obtains different cuts, such as ratio cut and normalized cut~\citep{hein2011beyond, von2007tutorial}.
The numerator of \eqref{eq:balanced_k_cut} can be expressed in matrix form. 
To see that, first one rewrites $\text{cut}(C_k, \bar{C}_k)$ as 
\begin{equation*}
    \label{eq:cut}
     \sum \limits_{i \in C_k, j \in \bar C_k} a_{ij}(1- z_i z_j) \;\; \text{s.t.} \;\; z_i, z_j \in \{ -1, 1 \},
\end{equation*}
Then,
\begin{equation*}
    \begin{aligned}
    & \sum \limits_{i,j} a_{ij}(1- z_i z_j) = \sum \limits_{i,j} a_{ij} \left(\frac{z_i^2 + z_j^2}{2} - z_i z_j \right) \\
    & = \frac{1}{2} \sum \limits_i \Bigg[ \sum \limits_j a_{ij} \Bigg]z_i^2 + \frac{1}{2} \sum \limits_j \Bigg[ \sum \limits_i a_{ij} \Bigg] z_j^2 \dots \\
    & - \sum \limits_{i,j} a_{ij} z_i z_j \\
    & = \frac{1}{2} \sum \limits_i d_{ii} z_i^2 + \frac{1}{2} \sum \limits_j d_{jj} z_j^2 - \z^T \A \z \\
    & = \z^T \D \z - \z^T \A \z = \z^T \Lb \z,
    \end{aligned}
\end{equation*}
where $\Lb$ is the graph Laplacian.
The relaxation done in spectral clustering to handle the discrete optimization problem is:
\begin{equation}
    \label{eq:general_sc_relax}
    \min_{\z_k \in \{ -1, 1 \}^{N}} \sum \limits_{k=1}^K \frac{\z_k^T \Lb \z_k}{\hat{B}(C_k)} \rightarrow \min_{\s_k \in \R^{N}} \sum \limits_{k=1}^K \frac{\s_k^T \Lb \s_k}{B(C_k)},
\end{equation}
where $B(C_k)$ is the continuous counterpart of $\hat{B}(C_k)$. 
\newline

\textit{Remark 1:}
Besides the graph Laplacian, other operators matching the sparsity pattern of the adjacency matrix can be used in problem~\eqref{eq:general_sc_relax}. One of such operators is the symmetrically normalized Laplacian, $\Lb_s = \I - \D^{-\frac{1}{2}} \A \D^{-\frac{1}{2}}$, which generally yields a different partition as the edges to be cut are weighted by the degree of their end nodes.
\newline

\textit{Remark 2:}
The term $\s^T \Lb \s$ in~\eqref{eq:general_sc_relax} measures the \textit{local quadratic variation} (LQV) of $\s$ on the graph, which is the quadratic variation of $\s$ across adjacent vertices.  
Laplacian smoothing minimizes LQV, by making similar the elements $s_i$ and $s_j$ if nodes $i$ and $j$ are connected. When using $\Lb_s$, the LQV is:

\begin{equation}
    \label{eq:lap_smooth_symm}
    \s^T \Lb_s \s = \frac{1}{2} \sum_{(i,j) \in \mathcal{E}} a_{i,j}\left(\frac{s_i}{\sqrt{d_i}} - \frac{s_j}{\sqrt{d_j}}\right)^2.
\end{equation}

%% -------------------------------------------------------
\subsection{Graph Neural Networks}
The main building block of a GNN is the message passing (MP) layer that, first, combines the node features with those of the neighbors on the graph.
Then, the aggregated features are mapped into a new representation by applying an affine transformation and a nonlinearity~\citep{gilmer2017neural}.
A basic MP layer is implemented as follows:
\begin{equation}
    \label{eq:mp}
    \X^{(l+1)} = \texttt{MP}(\X^{(l)}, \tilde \A) = \sigma(\tilde \A \X^{(l)} \boldsymbol{\Theta}_l),
\end{equation}
where $\X^{(l)}$ and $\X^{(l+1)}$ are, respectively, the input and output node features of the $l$-th MP layer, $\tilde \A$ is an operator matching the sparsity pattern of $\A$, $\sigma$ is a nonlinear activation function, and $\boldsymbol{\Theta}_l$ are trainable parameters.

%% -------------------------------------------------------
\subsection{Clustering with GNNs}
MinCutPool~\citep{bianchi2020spectral} is a GNN layer that computes soft cluster assignments as:
\begin{equation}
    \label{eq:soft_assign}
    \Sb = \texttt{softmax} \left( \texttt{MLP} \left( \bar \X, \boldsymbol{\Theta}_\text{MLP} \right) \right) \in \mathbb{R}^{N \times K},
\end{equation}
where $K$ is the number of clusters, $\bar \X$ are node features generated by a stack of one or more MP layers, and $\texttt{MLP}(\cdot)$ denotes a multi-layer perceptron with trainable parameters $\boldsymbol{\Theta}_\text{MLP}$.
The \texttt{softmax} function ensures that $\Sb$ is a proper cluster assignment matrix, since $\Sb \bm{1} = \bm{1}$ and $0 \leq s_{i,j} \leq 1$.

To learn the cluster assignments, MinCutPool optimizes the following unsupervised loss:
\begin{equation}
    \label{eq:mincutpool}
    \mathcal{L}_{mc} = 
    \underbrace{- \frac{\textrm{Tr} ( \Sb^T \tilde \A \Sb )}{\textrm{Tr} ( \Sb^T \tilde \D \Sb)}}_{\mathcal{L}_q} + 
    \underbrace{\bigg{\lVert} \frac{\Sb^T\Sb}{\|\Sb^T\Sb\|_F} - \frac{\I_K}{\sqrt{K}}\bigg{\rVert}_F}_{\mathcal{L}_b}, 
\end{equation}
where $\tilde{\A} = \D^{-\frac{1}{2}}\A\D^{-\frac{1}{2}}$ and $\tilde \D$ is the degree matrix of $\tilde \A$.
The first term, $\mathcal{L}_q$, minimizes the LQV, while $\mathcal{L}_b$ is a balancing term that helps prevent degenerate solutions.
Compared to problem~\eqref{eq:general_sc_relax}, the LQV and the balancing terms are summed rather than taking their ratio.
This helps both to prevent numerical issues when $\mathcal{L}_b$ gets too small and to keep $\mathcal{L}_{mc}$ in a controlled range, which is desirable when the GNN must also minimize other losses. 

Similarly to MinCutPool, DMoN~\citep{tsitsulin2020graph} optimizes a loss composed of an LQV and a balancing term:
\begin{equation}
\label{eq:dmon}
    \mathcal{L}_{dm} = 
    \underbrace{- \frac{\text{Tr}(\Sb^T\tilde\A\Sb)}{2E}}_{\mathcal{L}_m} 
    + 
    \underbrace{\frac{\sqrt{K}}{N} \left\| \sum_i \Sb_i^T \right\|_F -1}_{\mathcal{L}_r},
\end{equation}
where $\tilde\A = \A - \mathbf{d}^T\mathbf{d}$ and $\mathbf{d}$ is the degree vector of $\A$. The term $\mathcal{L}_m$ pushes strongly connected components to the same cluster, while $\mathcal{L}_r$ is a regularization term that penalizes the degenerate solutions.

%%%%%%%%%%%%%%%%%%%%%%%%%%%%%%%%%%%%%%%%%
% METHOD
%%%%%%%%%%%%%%%%%%%%%%%%%%%%%%%%%%%%%%%%%
\section{Proposed approach}

The cluster assignments $\Sb$, computed as in~\eqref{eq:soft_assign}, can be optimized by minimizing:
\begin{empheq}[box={\mymath[sharp corners]}]{equation}
    \label{eq:loss}
    \mathcal{L} = - \textrm{Tr} \left( \sqrt{\Sb^T\Sb} \right)
\end{empheq}
The proposed loss simplifies $\mathcal{L}_{mc}$ and $\mathcal{L}_{dm}$ considerably as it consists only of a balancing term.
Such a simplification offers the following advantages:
\begin{itemize}
    \item The computational complexity is reduced, as less operations are needed to compute $\mathcal{L}$.
    \item Fewer competing terms in the loss can ease the training and speed-up the convergence.
    \item There are no ratios in $\mathcal{L}$, which could cause numerical instability during training.
\end{itemize}

Despite its simplicity, the proposed loss can still yield an optimal clustering assignment.
The key insights that motivated its design are presented in the following.

%---------------------------------------------
\subsection{Removal of the LQV term}
\label{sec:lqv}
The absence of the LQV term in the loss is compensated by the presence of the MP layers that generate the features $\bar \X$ used to compute the cluster assignments $\Sb$ in \eqref{eq:soft_assign}.
In particular, consider the following MP layer:
\begin{empheq}[box={\mymath[sharp corners]}]{equation}
    \label{eq:mp_LQV}
    \small
    \X^{(l+1)} = \sigma\left( \left[\I - \delta (\I - \D^{-\frac{1}{2}}\A \D^{-\frac{1}{2}}) \right] \X^{(l)} \boldsymbol{\Theta}_l \right)
\end{empheq}
where $\tilde \A = \I - \delta (\I - \D^{-\frac{1}{2}}\A \D^{-\frac{1}{2}})$ is an operator matching the sparsity pattern of the graph and $\delta$ is an hyperparameter. 
When $\delta=1$, Eq.~\ref{eq:mp_LQV} reduces to $\X^{(l+1)} = \sigma(\D^{-1/2} \A \D^{-1/2} \X^{(l)} \boldsymbol{\Theta}_l)$, which is very similar to the update equation used in popular MP layers~\citep{defferrard2016convolutional}.
When $\delta=0$, the node features are not aggregated with those of the neighbors and the MP layer becomes equivalent to a dense layer of an MLP.

\begin{theorem}
The MP layer in \eqref{eq:mp_LQV} minimizes the LQV of the node features $\X$.
\end{theorem}
\begin{proof}
Let
\begin{equation*}
    ||\bm{x}||_{LQV} = \frac{1}{2}\sum_{i=1}^N \sum_{j=1}^N a_{i,j} \left(\frac{x_i}{\sqrt{d_i}} - \frac{x_j}{\sqrt{d_j}}\right)^2
\end{equation*}

To minimize the LQV, we first compute the derivative with respect to the $i$-th component:
\begin{align}
\label{eq:gradient_descent_LQV_alt}
(\partial||\bm{x}||_{LQV})_i &= \frac{1}{\sqrt{d_i}} \sum_j a_{i,j} \left( \frac{x_i}{\sqrt{d_i}} - \frac{x_j}{\sqrt{d_j}} \right) \nonumber\\
                             &= \frac{x_i}{d_i} \sum_j a_{i,j} - \sum_j \frac{a_{i,j}}{\sqrt{d_i}\sqrt{d_j}} x_j \nonumber\\
                             &= (\I \x)_i - (\D^{-1/2} \A \D^{-1/2} \x)_i \nonumber
\end{align}
The whole gradient is given by:
\begin{equation*}
    \nabla(||\bm{x}||_{LQV}) = (\I - \D^{-1/2} \A \D^{-1/2}) \bm{x}
\end{equation*}

The following update minimizes the LQV of the node features with gradient descent:
\begin{align*}
    \bm{x}^{(l+1)} & = \bm{x}^{(l)} - \delta  \nabla(||\bm{x}||_{LQV})^{(l)}\\
                   & = \bm{x}^{(l)} - \delta  (\I - \D^{-1/2} \A \D^{-1/2}) \bm{x}^{(l)}
\end{align*}
where $\delta$ indicates the gradient step.
\end{proof}

Clearly, there is a difference in minimizing the LQV of $\X$ rather than the LQV of $\Sb$ directly, as done by the terms $\mathcal{L}_q$ in \eqref{eq:mincutpool} and $\mathcal{L}_m$ in \eqref{eq:dmon}.
Nevertheless, being \eqref{eq:soft_assign} a smooth function, if two node features $\x_i$ and $\x_j$ are similar they will likely be mapped into similar cluster assignments $\s_i$ and $\s_j$.

%---------------------------------------------
\subsection{Optimality of the proposed balancing term}
The purpose of the balancing term is to encourage a partition where:
\begin{itemize}
    \item the nodes are assigned with high confidence to only one cluster;
    \item the number of elements in each cluster is approximately the same.
\end{itemize}

Typical degenerate solutions are those that violate one these two requirements. 
In particular, the first degenerate solution is when samples are uniformly assigned to all cluster, i.e., $\s_i = [1/K, 1/K, \dots, 1/K]$, $i=1, \dots, N$. The second degenerate solution occurs when all samples are assigned to the same cluster, e.g., $\s_i = [1, 0, \dots, 0]$, $i=1, \dots, N$.

It is straightforward to see that the first degenerate solution is avoided when $\mathcal{L}$ in \eqref{eq:loss} is minimized. 
The trace is the sum of the elements along the diagonal of $\sqrt{\Sb^T\Sb}$. 
Since $0 \leq s_{i,j} \leq 1$, the trace is maximized when each cluster assignment $\s_i$ has one entry equal to 1, i.e., when $s_{i,j} \in \{0,1\}$. 
On the other hand, if a sample $i$ is assigned to more than one cluster, due to the \texttt{softmax} normalization $\s_i$ will have two or more non-zero entries with values less than 1.
As a consequence, at least one value in the off-diagonal of $\sqrt{\Sb^T\Sb}$ would be non-zero and the trace would be smaller. 

The second degenerate solution is avoided when the clusters assume equal size, which is ensured by the following theorem.

\begin{theorem}
The optimum of the problem
\begin{equation}
    \label{eq:opt_constr}
    \max \limits_{\Sb \mathbf{1} = \mathbf{1}, s_{i,j} \in \{0,1\}} \rm{Tr} \left( \sqrt{\Sb^T\Sb} \right)
\end{equation}
is a balanced partition with clusters of size $\frac{N}{K}$.
\end{theorem}
\begin{proof}
Let $\Sb^T\Sb = \mathbf{C} \in \mathbb{R}^{K \times K}$.
The conditions $\Sb \mathbf{1} = \mathbf{1}$ and $s_{i,j} \in \{0,1\}$ imply the constraint $\sum_{i=1}^K c_{i,i} = N$, being $c_{i,i}$ the volume of samples assigned to cluster $i$.
The solution  of the constrained optimization problem can be found using Lagrange multipliers. Specifically, the Lagrangian is
\begin{equation*}
    L(\mathbf{C}, \lambda) = \sum_{i=1}^K \sqrt{c_{i,i}} - \lambda\left(\sum_{i=1}^K c_{i,i} - N \right).
\end{equation*}
To be an optimum, $\frac{\partial L(\mathbf{C}, \lambda)}{\partial c_{i,i}} = 0$ must hold for for each $i$. 
Hence, by taking the derivative one obtains
\begin{equation}
\label{proof_eq1}
    \frac{1}{2 \sqrt{c_{i,i}}} - \lambda = 0 \rightarrow c_{i,i} = \frac{1}{4 \lambda^2}
\end{equation}
Solving for $\lambda$ can be done by considering the constraint $\sum_{i=1}^K c_{i,i} = N$, which gives
\begin{equation}
\label{proof_eq2}
    \sum_{i=1}^K \frac{1}{4 \lambda^2} = N \rightarrow \lambda = \sqrt{\frac{K}{4N}}.
\end{equation}
Finally, substituting \eqref{proof_eq2} in \eqref{proof_eq1} gives $c_{i,i} = \frac{N}{K}$.
\end{proof}

%---------------------------------------------
\subsection{Numerical computation}
The matrix $\Sb^T\Sb$ is symmetric and is either positive definite or positive semi-definite
The latter case occurs when at least one cluster is completely empty, i.e., $s_{i,k}=0$ for each vertex $i$.
Either way, the matrix $\sqrt{\Sb^T\Sb}$ is unique and real.
The matrix square root is computed by first reducing the matrix to quasi-triangular form with the real Schur decomposition. The square root of the quasi-triangular matrix is then computed directly.

The computational cost of the proposed loss is dominated by the matrix multiplication $\Sb^T\Sb$ and by the Schur decomposition.
The first costs $\mathcal{O}(NK^2)$, while the second costs $\mathcal{O}(K^3)$. Since $N > K$, the total cost is $\mathcal{O}(NK^2)$.
On the other hand, the cost in MinCutPool and DMoN is dominated by the numerators of the LQV terms in \eqref{eq:mincutpool} and \eqref{eq:dmon}, whose complexity is $\mathcal{O}(N^2K + NK^2)$. When using sparse operations, the latter cost is reduced to $\mathcal{O}(EK + NK^2)$, which is still greater than the complexity of the proposed loss.

\section{Experimental evaluation}

\bgroup
\def\arraystretch{1} %vertical padding
\setlength\tabcolsep{1.0em} %horizontal padding
\begin{table}[!ht]
\footnotesize
\centering
\label{tab:cit_dataset}
\begin{tabular}{lcccc}
\cmidrule[1.5pt]{1-5}
\textbf{Dataset} & $N$ & $E$ & $F$ & $K$ \\
\cmidrule[.5pt]{1-5}
Cora     & 2,708  & 10,556   & 1,433 & 7 \\
Citeseer & 3,327  & 9,104   & 3,703 & 6 \\
Pubmed   & 19,717 & 88,648  & 500   & 3 \\
DBLP     & 17,716 & 105,734 & 1,639 & 4 \\
\cmidrule[1.5pt]{1-5}
\end{tabular}
\vspace{-.3cm}
\caption{\small Details of the citation datasets: number of nodes ($N$), number of edges ($E$), dimension of the feature vectors ($F$), and number of classes/clusters ($K$).} 
\end{table}
\egroup

The clustering performance are tested on four citation datasets, whose details are reported in Tab.~\ref{tab:cit_dataset}.
In each dataset, the number of clusters $K$ is set equal to the number of classes to which the nodes belong to.
The distribution of the node classes for each dataset, depicted in Fig.~\ref{fig:class_distr}, shows that the classes are rather unbalanced, which makes the clustering problem non-trivial.

\begin{figure}[!ht]
    \centering
    \includegraphics[width=.9\columnwidth]{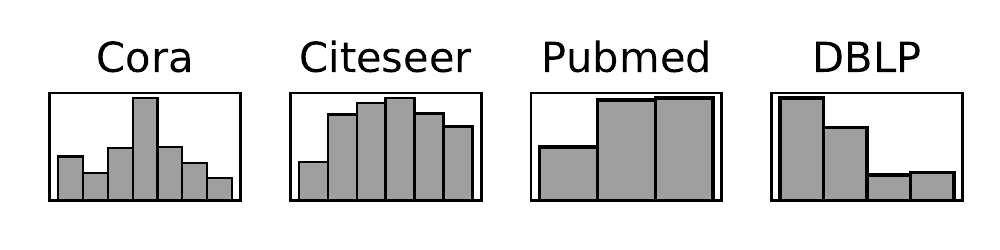}
    \vspace{-.8cm}
    \caption{\small Distribution of the node classes.}
    \label{fig:class_distr}
\end{figure}

The proposed architecture, referred to as \textit{Just Balance GNN} (JBGNN), consists of a stack of MP layers, which outputs $\bar \X$, followed by an MLP, which outputs $\Sb$.
The MP layers are implemented by a GCN~\cite{kipf2016semi} operating on the connectivity matrix $\tilde \A$ defined in \eqref{eq:mp_LQV}.
The whole JBNN is optimized only by minimizing the loss in \eqref{eq:loss}.
The following hyperparameters are used for each dataset:
$\delta$: 0.85;
MP activation: ReLU;
MP channels: 64;
MP layers: 10;
MLP activation: ReLU;
MLP channels: 16;
MLP hidden layers: 1;
learning rate: 5e-5;
epochs: 2,000.
The software implementation of JBNN is available online
\footnote{\url{https://github.com/FilippoMB/Simplifying-Clustering-with-Graph-Neural-Networks}}.

\bgroup
\def\arraystretch{1} %vertical padding
\setlength\tabcolsep{.3em} %horizontal padding
\begin{table*}[!ht]
    \centering
    \footnotesize
    \begin{tabular}{lcccccccc}
        \cmidrule[1.5pt]{1-9}
        & \multicolumn{2}{c}{\textbf{Cora}} & \multicolumn{2}{c}{\textbf{Citeseer}} & \multicolumn{2}{c}{\textbf{Pubmed}} & \multicolumn{2}{c}{\textbf{DBLP}}\\
        & ACC & NMI & ACC & NMI & ACC & NMI & ACC & NMI \\
        \midrule
        SC       & 0.298             & 0.028             & 0.217             & 0.014             & 0.589             & 0.182             & 0.458             & 0.023  \\
        DeepWalk & 0.229 {\t(0.020)} & 0.064 {\t(0.024)} & 0.193 {\t(0.002)} & 0.004 {\t(0.001)} & 0.361 {\t(0.001)} & 0.001 {\t(0.000)} & 0.266 {\t(0.001)} & 0.001 {\t(0.000)} \\
        Node2vec & 0.229 {\t(0.025)} & 0.060 {\t(0.029)} & 0.194 {\t(0.003)} & 0.004 {\t(0.001)} & 0.362 {\t(0.001)} & 0.001 {\t(0.000)} & 0.272 {\t(0.001)} & 0.001 {\t(0.000)} \\
        NetMF    & 0.389             & 0.251             & 0.277             & 0.127             & 0.448             & 0.058             & 0.455             & 0.037  \\
        \midrule
        GAE      & 0.464 {\t(0.062)} & 0.327 {\t(0.051)} & 0.381 {\t(0.038)} & 0.162 {\t(0.029)} & 0.588 {\t(0.071)} & 0.235 {\t(0.044)} & 0.416 {\t(0.035)} & 0.111 {\t(0.028)} \\
        VGAE     & \textbf{0.572} {\t(0.054)} & \textbf{0.437} {\t(0.028)} & 0.360 {\t(0.038)} & 0.156 {\t(0.034)} & \textbf{0.610} {\t(0.060)} & \textbf{0.24}5 {\t(0.043)} & 0.507 {\t(0.047)} & 0.212 {\t(0.021)} \\
        \midrule
        DiffPool & 0.472 {\t(0.010)} & 0.306 {\t(0.005)} & 0.336 {\t(0.007)} & 0.180 {\t(0.008)} & 0.418 {\t(0.002)} & 0.084 {\t(0.001)} & 0.370 {\t(0.042)} & 0.045 {\t(0.043)} \\
        DMoN     & 0.488 {\t(0.063)} & 0.357 {\t(0.042)} & 0.364 {\t(0.043)} & 0.196 {\t(0.029)} & 0.559 {\t(0.042)} & 0.192 {\t(0.048)} & 0.590 {\t(0.039)} & 0.334 {\t(0.026)} \\
        MinCut   & 0.534 {\t(0.041)} & 0.406 {\t(0.029)} & \textbf{0.497} {\t(0.049)} & \textbf{0.295} {\t(0.029)} & 0.572 {\t(0.034)} & 0.208 {\t(0.014)} & 0.538 {\t(0.033)} & 0.297 {\t(0.024)} \\
        JBGNN    & 0.457 {\t(0.025)} & 0.351 {\t(0.128)} & 0.334 {\t(0.019)} & 0.140 {\t(0.024)} & 0.564 {\t(0.023)} & 0.223 {\t(0.013)} & \textbf{0.607} {\t(0.008)} & \textbf{0.359} {\t(0.008)} \\
        \cmidrule[1.5pt]{1-9}
    \end{tabular}
    \vspace{-.3cm}
    \caption{\small Clustering performance metrics for each dataset. For method with stochastic components, the mean and the standard deviation (in brackets) obtained from 10 independent runs is reported.}
    \label{tab:results}
\end{table*}
\egroup

The performance of JBGNN are compared against three classes of methods. 
The first are algorithms that generate node embeddings based only on the adjacency matrix. The node embeddings are then clustered with $k$-means.
Representatives of this category are Spectral Clustering (SC), DeepWalk~\citep{perozzi2014deepwalk}, Node2vec~\citep{grover2016node2vec}, and NetMF~\citep{qiu2018network}.
The second class of methods are neural networks that generate node embeddings by accounting both for the adjacency matrix and for the node features. 
Also in this case, the learned embeddings are clustered with $k$-means in a post-processing step.
The chosen representatives for this category are the Graph AutoEncoder (GAE) and Variational Graph AutoEncoder (VGAE)~\citep{kipf2016variational}.
Finally, the last class of methods are end-to-end GNN models that directly generate soft cluster assignments $\Sb$ by accounting both for the graph connectivity and the node features. The hard cluster assignments are computed as $\bar \s = \texttt{argmax}(\Sb)$.
DiffPool~\citep{ying2018hierarchical}, DMoN~\citep{tsitsulin2020graph}, MinCutPool~\citep{bianchi2020spectral}, and the proposed JBGNN belong to this class.
To make the comparison fair between DiffPool, DMoN, MinCutPool, and JBGNN, the GNN architectures are configured to have the same capacity (number of layers and trainable parameters) and are trained for the same number of epochs. The only difference is in the loss and in how the cluster assignments $\Sb$ are computed (Diffpool uses an MP layer rather than the MLP in \eqref{eq:soft_assign} to compute $\Sb$).

The clustering performance is evaluated in terms of the normalized mutual information (NMI) between true class labels $\bm y$ and cluster assignments $\bar \s$ and the cluster accuracy (ACC), defined as
\begin{equation}
    \label{eq:clust_acc}
    \text{ACC} = \frac{1}{N} \sum \limits_{i=1}^N \theta \left(y_i, h(\bar s_i) \right),
\end{equation}
where $h(\cdot)$ maps the hard-cluster assignment $\bar s_i$ to the best matching class label according to Kuhn-Munkres algorithm, and $\theta$ is the Heaviside step function, which is 1 when $y_i$ and $h(\bar s_i)$ are equal, and zero otherwise.
The results are reported in Tab.~\ref{tab:results}.
Methods with stochastic components are trained and evaluated 10 different times.

\begin{table}[!ht]
    \centering
    \footnotesize
    \begin{tabular}{lcccc}
        \cmidrule[1.5pt]{1-5}
        & \textbf{Cora} & \textbf{Citeseer} & \textbf{Pubmed} & \textbf{DBLP} \\
        \midrule
        Diffpool   & 0.009                         & 0.030                        & 0.234                          & 0.229                         \\
           & {\color[HTML]{CB0000} \t +200\%} & {\color[HTML]{CB0000} \t +42\%} & {\color[HTML]{CB0000} \t +1376\%} & {\color[HTML]{CB0000} \t +477\%} \\
        DMoN       & 0.006                         & 0.028                        & 0.030                          & 0.061                         \\
                   & {\color[HTML]{CB0000} \t +100\%} & {\color[HTML]{CB0000} \t +33\%} & {\color[HTML]{CB0000} \t +76\%}   & {\color[HTML]{CB0000} \t +27\%}  \\
        MinCut & 0.006                         & 0.029                        & 0.030                          & 0.061                         \\
                   & {\color[HTML]{CB0000} \t +100\%} & {\color[HTML]{CB0000} \t +38\%} & {\color[HTML]{CB0000} \t +76\%}   & {\color[HTML]{CB0000} \t +27\%}  \\
        \midrule
        JBGNN      & \textbf{0.003}                         & \textbf{0.021}                        & \textbf{0.017}                          & \textbf{0.048}               \\
        \cmidrule[1.5pt]{1-5}
    \end{tabular}
    \vspace{-.3cm}
    \caption{\small Training times (seconds/step). In red, the increments in training times with respect to JBGNN.}
    \label{tab:times}
\end{table}

The computational complexity of the GNN models for clustering is measured in terms of training times.
Tab.~\ref{tab:times} reports the seconds used by each model to perform one step of gradient descent and, in red, the percentage increments from the times used by JBGNN.
Times are measured on an Nvidia RTX A6000.
\begin{figure}[!ht]
    \centering
    \subfigure[MinCutPool]{
        \includegraphics[width=\columnwidth]{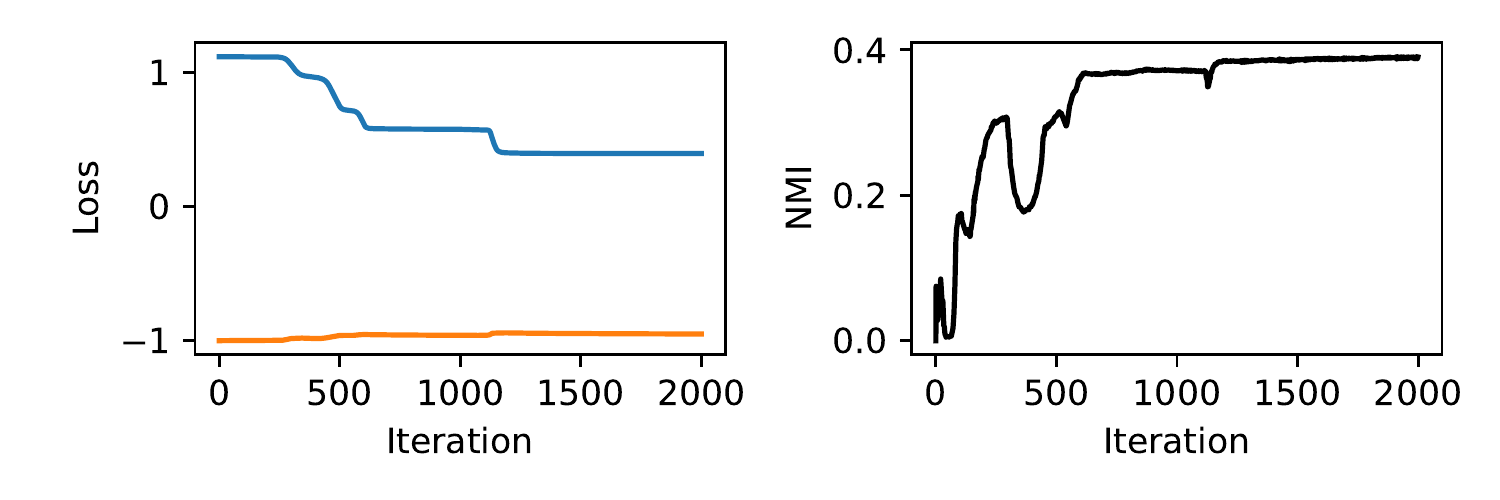}
    }\vspace{-.2cm}
    \subfigure[DMoN]{
        \includegraphics[width=\columnwidth]{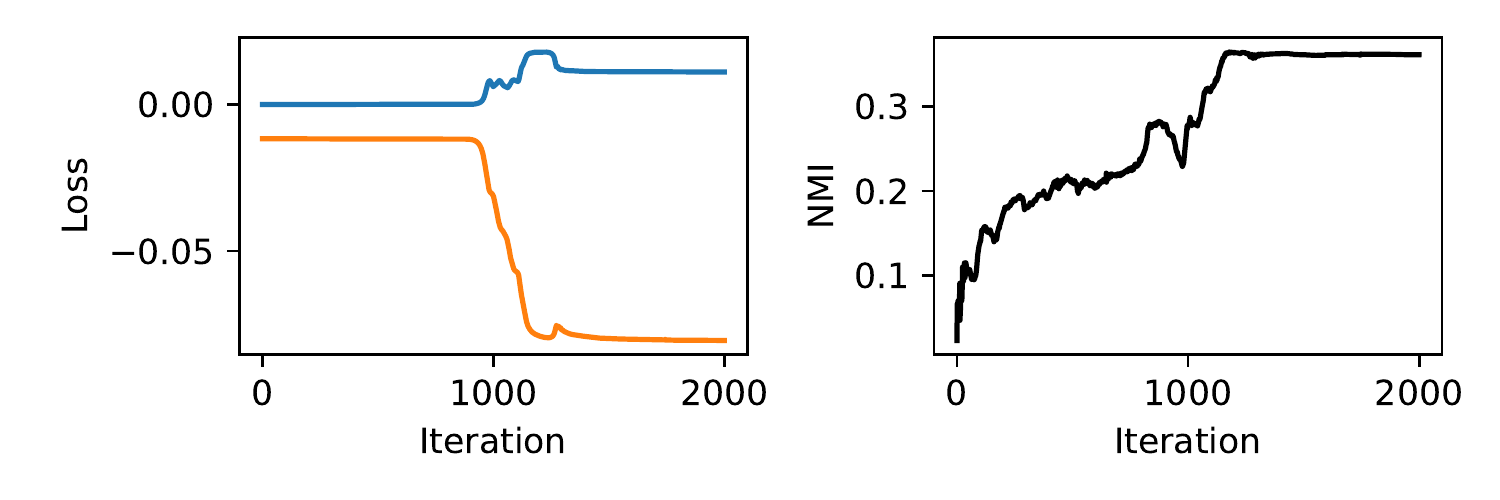}
    }\vspace{-.2cm}
    \subfigure[JBGNN]{
        \includegraphics[width=\columnwidth]{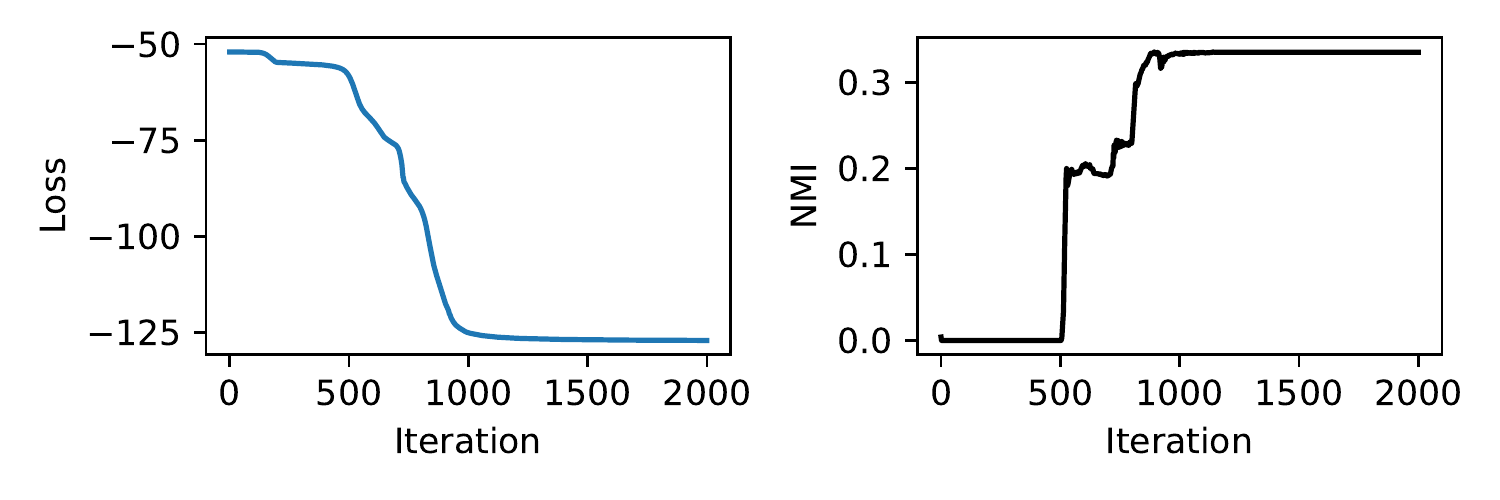}
    }
    \vspace{-.3cm}
    \caption{\small Evolution of the losses and NMI when training MinCutPool, DMoN and JBGNN on Cora.}
    \label{fig:losses}
\end{figure}
Finally, Fig.~\ref{fig:losses} reports the evolution of the losses and the NMI during training. 
LQV losses are in orange and balance losses are in blue.

\section{Discussion and conclusions}

The clustering performance of JBGNN is competitive with respect to the other end-to-end GNN clustering methods: it achieves a comparable ACC and NMI on Cora and Pubmed, and it is the best overall performing method on DBLP.
Diffpool is the worse-performing method in this category, as it does not optimize a proper clustering objective.
On Citeseer, MinCutPool outperforms every other method, including JBGNN, by a large margin. 
As discussed in Section~\ref{sec:lqv}, JBGNN minimizes the LQV of $\bar \X$ while MinCutPool directly minimizes the LQV of $\Sb$, which could make a significant difference in certain datasets.

About the performance of the other methods, SC, DeepWalk, Node2vec, and NetMF always achieve worse ACC and NMI. 
This is expected, as they build embeddings based only on the graph connectivity.
On the other hand, GAE and VGAE also account for the node features and achieve superior performance.
Compared to the end-to-end GNN clustering methods, GAE and VGAE follow a significantly different approach: they optimize a link reconstruction loss to learn node embeddings, which are then clustered in post-processing with $k$-means.
Despite the combination of GAE and VGAE embeddings with $k$-means yields very good performance, it has the disadvantage of being a two-step procedure. In addition, the $k$-means is computationally expensive and is not designed for out-of-samples predictions.
On the other hand, the GNN-based approaches directly output soft cluster assignments, which are efficiently evaluated at inference time and can also be used to implement differentiable operations, such as graph pooling~\citep{grattarola2022understanding}. 

In terms of computational complexity, JBGNN outperforms all other GNN clustering methods thanks to its simplicity and efficient formulation.
In particular, JBGNN achieves a speed improvement ranging from 27\% to  1,376\% compared to the other methods across the four datasets, making it particularly suitable for large-scale applications.
Referring to Fig.~\ref{fig:losses}, MinCutPool and DMoN start to converge between 1000 and 1500 epochs, while JBGNN generally converges earlier.
This, together with the faster updates, makes training the JBGNN significantly faster.

A natural extension of this work is to test JBGNN in other tasks, such as to implement graph pooling~\citep{grattarola2022understanding} in a deep GNN architecture for graph classification.
As a concluding remark, it should be noted that the proposed clustering objective relies on the assumption that all clusters have equal size.
This is reasonable, since clustering is an unsupervised task and the actual sizes are usually unknown.
However, if information about cluster size is available, a clustering objective that puts a bias towards an unbalanced partition could be used instead.

\paragraph{Acknowledgments}
I gratefully acknowledge the support of Nvidia Corporation with the donation of the two RTX A6000 GPUs used in this work.

\bibliographystyle{abbrvnat}
\bibliography{references}

\end{document}